\documentclass[11pt,lettersize]{article}

\usepackage{multicol}
\usepackage{graphicx,hyperref}
\usepackage{amsmath,amsthm}
\usepackage[utf8]{inputenc}
\usepackage{microtype}
\usepackage{colortbl,arydshln} 
\usepackage[charter,cal]{mathdesign}
\usepackage{XCharter}

\usepackage[mathscr]{euscript}
\usepackage[all]{xy}
\usepackage[x11names]{xcolor}
\usepackage{tikz}
\usetikzlibrary{patterns}
\usetikzlibrary{decorations.pathreplacing}
\usetikzlibrary{shapes.multipart}
\usetikzlibrary{matrix}
\usepackage{listings}
\usepackage{mathtools}
\usepackage[portrait,letterpaper,margin=1in]{geometry}
\usepackage{verbatim}
\usepackage{subcaption}

\usepackage{algorithmicx}
\usepackage{algorithm}
\usepackage{algpseudocode}

\usepackage{braket}
\usepackage{todonotes}

\newcommand{\rset}{\mathbb R}

\def\Proj{\mathbb{P}}

\newcommand{\ep}{{\varepsilon}}

\let\R\Real
\let\C\Complex

\def\E{\operatorname{E}}			

\def\rank{\operatorname{rank}}

\DeclareMathOperator{\Span}{span}
\DeclareMathOperator{\diag}{diag}

\def\Ceil#1{\left\lceil #1 \right\rceil}

\def\Set#1{\left\{ #1 \right\}}
\def\Abs#1{\left| #1 \right|}
\def\Card#1{\left| #1 \right|}
\def\Norm#1{\left\| #1 \right\|}
\def\Paren#1{\left( #1 \right)}		
\def\Brack#1{\left[ #1 \right]}		

\makeatletter
\def\Bigbar#1{\mathrel{\left|\vphantom{#1}\right.\n@space}}

\makeatother

\definecolor{codegreen}{rgb}{0,0.6,0}
\definecolor{codegray}{rgb}{0.5,0.5,0.5}
\definecolor{codepurple}{rgb}{0.58,0,0.82}
\definecolor{backcolour}{rgb}{0.95,0.95,0.92}

\usepackage{enumitem}
\newlist{subproblems}{enumerate}{1}
\setlist[subproblems]{label=(\alph*)}

\def\E{\operatorname{\mathbb{E}}}

\newtheorem{theorem}{Theorem}
\newtheorem{lemma}[theorem]{Lemma}
\newtheorem{corollary}[theorem]{Corollary}

\theoremstyle{definition}
\newtheorem{definition}[theorem]{Definition}
\newtheorem{observation}[theorem]{Observation}

\theoremstyle{remark}

\newenvironment{proofof}[1]{ {\noindent \em Proof of #1.}\/}{\hfill\qedsymbol\bigskip}

\def\tpose{\mathsf{T}}

\newcommand{\BB}[1]{{B^{(#1)}}}

\renewcommand{\Proj}[1]{{\Pi^{(#1)}}}

\def\mm{{\mathbf{m}}}
\def\MM{{\mathbf{M}}}
\def\AA{{\mathbf{A}}}
\def\BB{{\mathbf{B}}}
\def\One{{\mathbb{1}}}

\DeclareMathOperator{\mom}{g}
\DeclareMathOperator{\emom}{{\tilde{g}}}

\def\calA{\mathcal{A}}
\def\calB{\mathcal{B}}

\newcommand{\CC}{\mathbf{C}}
\newcommand{\eCC}{{\tilde{\mathbf{C}}}}

\newcommand{\ev}{\tilde{v}}
\newcommand{\eu}{\tilde{u}}
\def\CBA{\CC_{\calB\calA}}
\def\eCBA{\eCC_{\calB\calA}}
\def\CBAp{\CC_{\calB'\calA}}
\def\eCBAp{\eCC_{\calB'\calA}}
\def\CBBA{\CC_{\calB+\calB', \calA}}
\def\eCBBA{\eCC_{\calB+\calB', \calA}}

\newcommand{\Normi}[1]{\Norm{#1}_{\infty}}
\newcommand{\tildecal}[1]{\tilde{\mathcal{#1}}}

\newcommand{\suppress}[1]{}
\renewcommand{\Proj}[1]{{\mathcal{P}_{#1}}}
\newcommand{\ProjPerp}[1]{{\mathcal{P}^{\perp}_{#1}}}
\newcommand{\be}{\begin{equation}}
\newcommand{\ee}{\end{equation}}
\newcommand{\bea}{\begin{eqnarray}}
\newcommand{\eea}{\end{eqnarray}}
\newcommand{\bean}{\begin{eqnarray*}}
\newcommand{\eean}{\end{eqnarray*}}

\DeclareMathOperator{\LearnPowerDistribution}{\textsc{LearnPowerDistribution}}
\DeclareMathOperator{\Vandermonde}{Vdm}

\def\eps{\varepsilon}
\def\given{\mid}

\begin{document}
\title{Source Identification for Mixtures of Product Distributions}

\author{
Spencer L. Gordon\thanks{Engineering and Applied Science, California Institute of Technology, {\tt slgordon@caltech.edu}.} \and
Bijan Mazaheri\thanks{Engineering and Applied Science, California Institute of Technology, {\tt bmazaher@caltech.edu}} \and
Yuval Rabani\thanks{The Rachel and Selim Benin School of Computer Science and Engineering, The Hebrew University of Jerusalem, Jerusalem 9190416, Israel, {\tt yrabani@cs.huji.ac.il}. Research supported in part by NSFC-ISF grant 2553-17 and by NSF-BSF grant 2018687. Part of this work was done while visiting Caltech.} \and
Leonard J. Schulman\thanks{Engineering and Applied Science, California Institute of Technology, {\tt schulman@caltech.edu}. Research supported in part by NSF grants CCF-1618795, 1909972.}
}

\maketitle

\begin{abstract}
We give an algorithm for source identification of a mixture of $k$
product distributions on $n$ bits. This is a fundamental problem in machine
learning with many applications. Our algorithm identifies the source parameters 
of an identifiable mixture, given, as input, approximate values of multilinear moments (derived, for 
instance, from a sufficiently large sample), using $2^{O(k^2)} n^{O(k)}$
arithmetic operations. Our result is the first explicit bound on the computational
complexity of source identification of such mixtures. The running time improves
previous results by Feldman, O'Donnell, and Servedio (FOCS 2005) and Chen
and Moitra (STOC 2019) that guaranteed only learning the mixture (without parametric identification of the source). Our
analysis gives a quantitative version of a qualitative characterization of 
identifiable sources that is due to Tahmasebi, Motahari, and Maddah-Ali (ISIT 2018).
\end{abstract}

 \thispagestyle{empty}
 \newpage
 \setcounter{page}{1}


\section{Introduction}\label{sec: intro}

The main result of this paper is an algorithm for source identification of finite 
mixtures of binary product distributions. Our algorithm identifies the source
parameters of a mixture of $k$ product distributions on $n\ge 3k-3$ 
observable bits that satisfy a sufficient condition for identifiability, which is that at least
$3k-3$ of the observable bits are $\zeta$-separated (definition below). The
algorithm uses, as inputs, empirical frequencies that are assumed to deviate from their 
true value by an error of at most $\eps < \zeta^{O(k^2\log k)}$. (This assumption will be satisfied, with high probability, for a sample of size roughly $1/\eps^2$.) The algorithm 
identifies the parameters of the mixture model to an accuracy of
$\zeta^{-O(k^2\log k)}\eps$. This is also roughly (up to a different constant hidden by
the big-Oh notation) the statistical distance between the empirical distribution and the output (i.e., the learned) distribution. The runtime of our algorithm
is $2^{O(k^2)} n^{O(k)}$ arithmetic operations.\footnote{We note that the runtime 
relates to the post-sampling computation, after aggregating the empirically observed 
frequencies. There are good reasons for making this distinction. Collecting the sample 
and computing the frequencies is computationally trivial. It can often be done under 
a streaming model or in parallel. Or the frequencies might be available from an 
external source.} \footnote{In this introduction we suppress dependence on mixture weights.}
Under a stronger assumption that \textit{all} observable bits are
$\zeta$-separated, the runtime improves to $2^{O(k^2)} n$.

To formulate the problem, consider observable random variables 
$X_1,\ldots,X_n$ that are distributed on a common range $R$. In a finite mixture 
model, the joint distribution on these random variables is governed by a {\em hidden} 
or {\em latent} random variable $H$ supported on $\{1,\ldots,k\}$, such that 
$X_1,\ldots,X_n$ are statistically independent conditional on $H$. We consider
the case of a finite range $R$. The hardest and most fundamental case is when the range is binary (i.e., $R=\{0,1\}$) and there are no further 
constraints relating the distributions at the observables; the case of larger $R$ 
reduces to this case (see~\cite{FOS08}, for example).

Finite mixture models were
pioneered in the late 1800s in~\cite{Newcomb86,Pearson94} in the context of 
applications in astronomy and the mathematical theory of evolution. It is difficult
to do justice to the vast literature in statistics on mixture models; see, e.g., the 
surveys~\cite{Everitt1981,TSM85,lindsay1995mixture,McLachlanLR19}. The 
computational complexity of learning mixture models was studied starting with 
the seminal papers~\cite{KMRRSS94,CGG01,Dasgupta99,FM99}. The machine 
learning community has shown recent interest in learning mixtures of product 
distributions for pattern recognition. Motivating applications abound in population 
genetics, bioinformatics, image recognition, text classification, and other areas, 
e.g.,~\cite{Pritchard945,JWLWC05,juan2004,juan2002use}. The algorithms in this 
context are primarily based on the method of iterative Expectation Maximization 
(EM) clustering (e.g.,~\cite{JGV04,LWPA16,PKM16,carreira2000practical}). As 
detailed further below, in this genre provable guarantees of source identification 
are provided in~\cite{najafi2020reliable}, but the runtime depends exponentially 
on the sample size, and a very large $n$ is required. In the theory of computing 
literature, special cases and variants of learning mixtures of product distributions 
were considered 
in~\cite{CGG01,FM99,CR08,AGM12,AFHKL12,AHK12,RSS14,LRSS15,KKMMR18,CM19,GMRS20}.
A more detailed account on the previous results~\cite{FOS08,CM19} on mixture learning is given below.

The question of \textit{source identification} dates back in the statistics literature at least to 
1950~\cite{KR50,Koopmans50,Teicher63,Blischke64,YS68}. This is still a thriving
area of research (see,
e.g.,~\cite{carreira2000practical,AMR09,VS15,TahmasebiMM18,RVS20,aragam2020}). 
From a computational 
perspective, the ``learning'' vs.\ ``source identification'' contrast was raised in~\cite{FM99}, in 
the context that interests us here of learning a mixture of several (in that work, two) binary product distributions.
``Learning'' means computing any hypothetical model that generates observable statistics
close to the empirical ones. ``Source identification'' means computing a model that
is close in parameter space to the true model underlying the empirical statistics.
Clearly, source identification is a more challenging goal that implies also learning.
Source identification is desirable for a variety of reasons. A key reason is the danger
of overfitting a model to the empirical data, thereby ruining its predictive guarantees (see,
for instance,~\cite{KF09} ch.~16). Another reason is that source identification is necessary for
the quantification of causal relations between the hidden variable and the observable variables
(see~\cite{pea09,SpGlS00} for an introduction to graphical causal models)---here identification, but not learning, provides explanatory value along with the possibility of effective intervention. Finally, whether in the causal context or otherwise, identification provides the statistician with \textit{semantics}---an actual characterization of the process generating the data---which may be far more useful than the mere ability to artificially generate samples from the same distribution.

The distinction between learning and source identification has two main aspects. First, 
a distribution on observable variables might be inherently explainable by two (or more) 
far-apart models, ruling out source identification. Thus, we must restrict our attention to
a class of {\em identifiable} models, for which the mapping of model to distribution on
observables is one-to-one. In this paper, we consider the case that there are at least 
$2k-1$ observables that are $\zeta$-separated. An observable bit $X$ is \textit{separated} if
the $k$ conditional probabilities $\Pr[X=1\mid H=j]$, $j=1,\dots,k$ are mutually distinct, and  
\textit{$\zeta$-separated} if every two values differ by at least $\zeta>0$.\footnote{Notice that
$\zeta$ cannot exceed $\frac{1}{k-1}$.}
Having $2k-1$ observables that are separated is a sufficient, and in general necessary
condition for identification~\cite{TahmasebiMM18}. (Depending on the parameters, 
fewer sources may suffice. But for example $2k-1$ are always necessary when the $X_i$ 
are iid conditional on $H$, which we call the ``power distribution'' case~\cite{RSS14}, since terms of the form $\prod_{i=1}^n \Pr[X_i=1\mid H=j]$ are replaced by $(\Pr[X_1=1\mid H=j])^n$.)

Second, even if a model is identifiable in the limit of perfect statistics (infinite sample size), 
 the available empirical statistics might be insufficiently accurate, yet still allow for the learning objective. 
It is obvious, though, that a learning algorithm that runs on data produced
by an \textit{identifiable} source and is required to achieve sufficiently high accuracy also implicitly
identifies the source.\footnote{This is because the space of source
parameters is compact and the mapping to observable statistics is
continuous.} 
However, a learning algorithm might learn
a mixture model without 
this implying that the empirical data enables identification.
In fact, the ground-breaking work of~\cite{FOS08} gives an $(nk/\eps)^{O(k^3)}$ time algorithm for {\em learning}
a model of $k$-mixture of $n$ binary product distributions that generates a distribution on the
observables within statistical distance $\eps$ of the empirical statistics. A faster $k^{O(k^3)}(n/\eps)^{O(k^2)}$
time learning algorithm is given in~\cite{CM19}.\footnote{Notice, in particular, footnote 2 in that paper.}
The algorithms in these papers are not guaranteed, under any assumptions, to identify the source to any particular accuracy $\eps$; in fact the algorithms succeed even in cases where the source is not identifiable.

Our contribution is twofold. First, we prove a quantitative 
version of the identifiability criterion: namely, for any given $\zeta>0$, we establish (roughly; 
see later for exact statements) that any two models which differ by $\eta$ in parameter space, 
differ in their statistics by at least $\eta \zeta^{O(k^2 \log k)}$. \footnote{To simplify the informal 
discussion, we state most of the bounds in the rest of this paragraph just for $\eta\ge \zeta^{O(k^2 \log k)}$.}
This, of course, implies that the algorithms of~\cite{FOS08,CM19}
can be used for source identification, assuming $\zeta$-separation and sufficiently small
target accuracy $\eps = \zeta^{O(k^2 \log k)}$. Second, we improve substantially over the runtime of these two 
algorithms, so that (even under the conditions under which our result imples that those algorithms can perform identification), our source identification algorithm is more efficient. Specifically, if we have at least $3k-3$ observables that are $\zeta$-separated, our algorithm requires 
empirical statistics accuracy $\zeta^{O(k^2 \log k)}$ (i.e., sample size $\zeta^{-O(k^2 \log k)}$) and has a
runtime of $2^{O(k^2)} n^{O(k)}$ arithmetic operations. Our algorithm can also identify the source using 
the minimum of $2k-1$ \, $\zeta$-separated observables; but then we require input accuracy 
$\zeta^{O(k^3)}$ (but the same $2^{O(k^2)} n^{O(k)}$ runtime). If \textit{all} observables are $\zeta$-separated, then the runtime 
improves to $2^{O(k^2)} n$ \footnote{Here we suppress a $O(\log \log (\eps^{-1}))$ term.} (with the same input accuracies, according to the number of variables). Our contributions establish quantitative bounds on the qualitative 
sufficiency of $\zeta$-separation of~\cite{TahmasebiMM18}; the results in that paper are entirely 
non-algorithmic. The only explicit algorithmic result on source identification that we are aware of 
is~\cite{najafi2020reliable}. Their algorithm, which runs under a somewhat more general assumption 
than $\zeta$-separation, requires complete enumeration over the choice of mixture constituent that
generated each sample point, for a sufficiently large sample. Thus, it is prohibitively expensive, requiring 
at least $n = \exp(k^2)$ observable random variables (as compared with $3k-3$), and runtime that is 
doubly exponential, namely $k^{\exp(k^2)}$.

The main idea underlying our algorithm is the following. Given sufficiently many $\zeta$-separated 
observables, for which we have sufficiently accurate empirical multilinear moments,\footnote{We call these empirical moments as we expect them to obtained by sampling; but our theorems depend only on their being sufficiently accurate.}
we
show how to construct ``synthetic bits'' for which we can compute highly accurate power moments, i.e., the moments that occur in the far more restricted problem of power distributions. 
  The higher moments of these synthetic bits are created out of linear 
combinations of multilinear moments of the original bits. This mechanism in its idealized form (i.e., 
for perfect statistics) suffices to re-prove the theorem of~\cite{TahmasebiMM18} that $2k-1$ 
separated observables suffice for source identification. The next challenge we face is to bound 
the coefficients of the multilinear monomials in these linear combinations, as this affects the required accuracy of the empirical statistics (thus, the
 required 
sample size). The synthetic bits method reduces the 
problem to the special case of identification of a mixture of $k$ power distributions, i.e., when the 
$X_i$-s are iid conditional on $H$. This special case is effectively an extension of the theory of orthogonal 
polynomials on the reals and the classical {\em moment problem}~\cite{schmudgen2017moment,Simon15}, 
and methods such as Prony's method or the Matrix Pencil method were shown to solve 
it~\cite{RSS14,LRSS15,KKMMR18,GMRS20}. Despite the power distributions case being so highly 
constrained, it has useful applications, e.g., in reconstructing population histories and in learning topic 
models (see the above references). We can use these existing algorithms to recover the mixture of power 
distributions on synthetic bits. That in turn enables recovery of the mixture of product distributions on the 
observable bits. The best algorithm for power distributions to date~\cite{GMRS20} requires estimates of 
the first $2k$ moments of the synthetic bits to within accuracy $\zeta^{O(k)}$, and has runtime $O(k^{2+o(1)})$. 
Thus, this component of our runtime is cheap; the runtime bound of our algorithm is dominated by the exhaustive 
search for $3k-3$ observables that are $\zeta$-separated (unless all observables are known to be $\zeta$-separated), and by the construction of the $2k$ synthetic bits. 

\section{Preliminaries and main theorem}\label{sec: pre}

\paragraph{Notation} There is a hidden variable $H$ ranging in $[k]$, and $n$ binary observable variables $X_i$; we write 
$\mm_{ij}:= \Pr[X_i=1\mid H=j]$.

Vectors are row vectors unless otherwise indicated. For $S \subseteq [n]$ define the random variable 
$X_S = \prod_{i\in S} X_i$. 
We make extensive use of \textit{Hadamard product} for 
vectors $u=(v_1,\ldots,v_k)$, $v=(v_1,\ldots,v_k)$:
\begin{align*}\odot: \rset^k \times \rset^k & \to \rset^k \\ u \odot v &= (u_1v_1,\ldots,u_kv_k) \end{align*}
The identity for this product is the all-ones vector $\One$. We associate with vector $u$ the linear
operator $u_\odot = \diag(u)$, a $k \times k$ diagonal matrix, so that 
\[ v \cdot u_\odot =u \odot v .\]

Let $\mm_i$ be the row vector $(\mm_{i1},\ldots,\mm_{ik})$. 
Let $\MM \in \R^{2^{[n]} \times k}$ be the matrix with rows indexed by subsets $S=\Set{i_1,\dotsc, i_s} \subseteq [n]$, with rows 
\[ \MM_S = \mm_{i_1} \odot \mm_{i_2} \odot \dotsb \odot \mm_{i_s}. \] 
In particular, $\MM_{\emptyset} = \One$ and $\MM_{\Set{i}} = \mm_i$ for all $i \in [m]$. 

Observe that source identification is not possible if $\MM$ has less than full column rank, i.e., $\rank \MM<k$, as then the mixing weights cannot be unique.

For a collection of subsets $\mathcal{S} \subseteq 2^{[n]}$, let $\MM[\mathcal{S}]$ denote the restriction of $\MM$ to the rows $\MM_S, S \in \mathcal{S}$. E.g., $\MM=\MM[2^{[n]}]$.
 
 \paragraph{The empirical multi-linear moments}
 For a finite sample drawn from the model, we let $\emom(S)$ be the empirical estimate of $\E[X_S]$, i.e., the fraction of samples for which $\prod_{i\in S}X_i=1$. These $\emom(S)$ for $S \subseteq [n]$ are the complete list of ``observables'' of the model. Each converges, in the infinite-sample limit, to the value $\mom(S):=\E[X_S]$,
\[ \mom(S)= \MM_S \pi^{\tpose} = \MM_S \pi_\odot \One^{\tpose}. \]

\paragraph{Comparison to the iid case}
The fact that we have access \textit{only} to multi-linear moments is the key constraint of the problem.
Compare with the ``power'' case, i.e., when we know in advance that all rows $\mm_i$ are identical. (That is, the $X_i$ are iid conditional on $H$.) Then observing bit $X_2$ alone tells us nothing about $X_1$ that we do not learn from observing bit $X_1$. But observing the product $X_1X_2$ does tell us something new: in fact it is distributed as would be a binary observable whose row had entries $(\mm_{11}^2,\ldots,\mm_{1k}^2)$. With $2k$ rows, we obtain the first $2k$ moments of the $k$-sparse distribution corresponding to $\mm_1$ (i.e., the real-valued, $k$-sparse distribution which places atomic probability $\pi_j$ at $\mm_{1j}\in \rset$, also called a ``$k$-spike'' distribution). From this point on, one may apply the time-honored method of Prony to identify the source. For details, a runtime analysis, and further references, see~\cite{GMRS20}. 

Much of the interest of the present problem, by contrast, is due precisely to the fact that we \textit{cannot} read out higher moments of the distributions corresponding to any of the bits $X_i$, because the various rows do not have any assumed relationship. How to nonetheless obtain higher moments of individual rows, is the challenge our algorithm tackles. 
 
\paragraph{Main theorem}
In what follows $\zeta$ is an assumed separation parameter, and $\pi_{\min}$ is an assumed lower bound on mixture weights.
 \begin{theorem} \mbox{}
 
($i$) Given access to the joint statistics of $n$ observable bits among which at least $3k-3$ which are $\zeta$-separated, with all statistics available to additive accuracy $\eps$ for $\eps\leq (\pi_{\min})^{O(\log k)} \zeta^{O(k^2 \log k)}$, our algorithm runs in time $2^{O(k^2)}n^{O(k)}$ and computes the model parameters ($\pi$ and all row values $\mm_{ij}$) to within accuracy $\eps  \zeta^{-O(k^2 \log k)}(\pi_{\min})^{-O(\log k)}$.

($ii$) If all rows are $\zeta$-separated, then a slightly simpler version of our algorithm identifies the source (to within the same accuracy, given the same input accuracy, as in ($i$)), in runtime $2^{O(k^2)}n$ (i.e., linear in $n$).

($iii$) If only $2k-1$ \, $\zeta$-separated rows are available, another simpler version of our algorithm also identifies the source, but the loss factor on the accuracy is $ (\pi_{\min})^{-k}\zeta^{-O(k^3)}$ and consequently one must start with $\ep\leq  (\pi_{\min})^{k}\zeta^{O(k^3)}$ which requires sample complexity comparable to prior work, but achieving the same improved runtime as in ($i$). 
 \end{theorem}

\section{Algorithm} \label{sec:alg}

\paragraph{Further definitions}
 Let $V_{(i)}$ be the subspace of $\rset^k$ spanned by the $i$'th standard basis vector,
 and $P_{(i)}$ the projection onto $V_{(i)}$ w.r.t.\ usual inner product (i.e., as a matrix, all $0$'s except a $1$ in entry $(i,i)$).
 For a separated vector $v=(v_1,\ldots,v_k)$ 
define the Lagrange interpolation polynomials $p_{v,i}$ by $p_{v,i}(v_j)=\delta_{ij}$. We then have the matrix equation \[ p_{v,i}(v_\odot)=P_{(i)}. \] Write $p_{v,i}(x)=\sum_{j=0}^{k-1} p_{v,i,j}x^j$. Define the ``coefficient norm'' of a polynomial by
$\|p_{v,i}\|_c=\sum_0^{k-1} j |p_{v,i,j}|$.

It will be very useful to put our observables in matrix form. 
Let $S,T\subseteq [n]$ be disjoint sets and take any $\mathcal{A}\subseteq 2^S, \mathcal{B} \subseteq 2^{T}$. Then the matrix 
$\CBA$ is observable (meaning every entry of it is a function of the joint statistics of the observable random variables $X_1,\ldots,X_n$), where
\be
\CBA \coloneqq \MM[\calB] \pi_\odot \MM[\calA]^\top. \label{CCBA}
\ee
Let $\eCBA \coloneqq \Brack{\emom(B \cup A)}_{B \in \calB, A \in \calA}$ be the corresponding matrix of empirical moments.

We'll use $\Norm{\cdot}$ to denote the operator norm for matrices (with respect to the Euclidean norm in the domain and range).

\paragraph{Strategy} If $\mm$ possesses $2k-1$ or $3k-3$ (depending on the version of the algorithm; this affects only runtime) $\zeta$-separated rows of $\mm$, and given empirical statistics within the required distance from perfect statistics, our algorithm will identify the model (to within similar accuracy). This includes even rows which are not themselves $\zeta$-separated; all we need is that some $2k-1$ or $3k-3$ rows be $\zeta$-separated.
The algorithm has the following structure: range over all $n^{O(k)}$ subsets of the rows; run the identification algorithm using that set. If the set does not suffice for identification (which can happen only if the set includes some non-$\zeta$-separated rows), this will be flagged by the algorithm. 
Any two such runs which do terminate successfully, must result in very close parameter reconstructions. 

\subsection{Synthetic bits and bootstrapping}
We start with the $2k-1$-observables version of the algorithm. Our final algorithm in section~\ref{sec:3k-3}, which uses $3k-3$ observables and achieves better runtime, will use a slightly more complicated construction, but the main ideas are present in the simpler variant here.

\subsubsection{Constructing higher moments of a row, using $2k-2$ other rows}
In what follows we show how to compute moments of arbitrary degree of the $k$-spike distribution associated with any bit $X_i$, given access to any additional $2k-2$ \, $\zeta$-separated observable bits. For concreteness, let $X_1$ (corresponding to the row $\mm_1$) be the variable for which we want to find higher moments. (We don't require that $\mm_1$ be $\zeta$-separated for the moment computation, although that will be needed in a subsequent step of the algorithm.) Let $S=\{2,\ldots,k\}$, $T=\{k+1,\ldots,2k-1\}$ be the indices of the $2k-2$ other $\zeta$-separated rows, partitioned into two sets of $k-1$ rows each. 

The only thing that our statistics tell us about row $\mm_1$, alone, is its first moment: $\E(X_1)=\mm_1 \pi_\odot \One^\top$. 
Equivalently this quantity is also the expectation the $k$-spike distribution associated with $X_1$. It will be critical to obtain higher moments of this distribution. The second moment  is equal to $(\mm_1 \odot \mm_1) \pi_\odot \One^\top$, and more generally (with $\mm_1^{\odot r}$ denoting the $r$-fold Hadamard product of $\mm_1$ with itself), the $r$'th moment is $\mm_1^{\odot r}\pi_\odot \One^\top$ for any $r$. We will need to have the $1,\dotsc,2k$'th moments in order to solve for $\mm_1$ and $\pi$.

We show in Section~\ref{sec:condition} that there exist subsets $\calA = \Set{A_1,\dotsc,A_k}\subseteq 2^S$, $\calB = \Set{B_1,\dotsc,B_k} \subseteq 2^T$ of size $k$ each such that $\AA \coloneqq \MM[\calA] \in \R^{k\times k}$ and $\BB  \coloneqq \MM[\calB]\in \R^{k\times k}$ are invertible. Moreover, we'll have $A_1 = \emptyset = B_1$ so that the first row of $\MM[A]$ and $\MM[B]$ will be the vector $\One$. Then the matrix $\CBA = \BB \pi_{\odot} \AA^{\tpose}$ (as defined in~\eqref{CCBA}) is an invertible, observable matrix.
 
Now consider the vector 
\[ v_1 \coloneqq \mm_1 \pi_\odot \AA^\tpose = (\E[X_1X_{A_1}],\dotsc,\E[X_1X_{A_k}]). \] Each coordinate $\E[X_1X_{A_i}] = \E[X_{A_i\cup\Set{1}}]$ is a multi-linear moment and is therefore observable. In the algorithm, our starting point will be approximations to multi-linear moments, and we will address that shortly, but for now we write down the algorithm as it would run with perfect statistics.

While we'd like to find $\mm_1^{\odot 2}$ directly, we won't be able to do so, and instead we'll settle for computing the vector $u_1$ satisfying 
\[u_1\BB \coloneqq \mm_1.\] This $u_1$ is unique since $\BB$ is invertible, and it gives the coefficients needed to express $\mm_1$ as a linear combination of the rows of $\BB$. We can in fact compute $u_1$ from observables, because \begin{equation}u_1 = v_1\CBA^{-1}. \label{u1v1CBA} \end{equation} We can verify~\eqref{u1v1CBA} algebraically: 
\begin{equation} v_1 (\CBA)^{-1} \BB = v_1 (A^\tpose)^{-1} \pi_\odot^{-1} \BB^{-1} \BB = 
v_1 (\AA^\tpose)^{-1} \pi_\odot^{-1} = \mm_1 \pi_\odot \AA^\tpose (\AA^\tpose)^{-1} \pi_\odot^{-1} =  \mm_1. \label{eq:u1B} \end{equation}

Conceptually what~\eqref{eq:u1B} means is that $u_1\BB$, a linear combination of rows of $\BB$, defines  a random variable that has been synthesized out of the $X_{B_\ell}$ random variables, and shares the same expectations as $X_1$ conditional on any setting of $H$. We call this the method of \textit{synthetic bits.} Our algorithm consists of repeatedly \textit{bootstrapping} synthetic bits.

To be more explicit, use $u_1=(u_{11},\dotsc,u_{1k})$ to define a new random variable $Y \coloneqq  \sum_{\ell=1}^k u_{1\ell} X_{B_\ell}$. 

$Y$ has been \textit{synthesized} out of $X_{B_1},\dotsc,X_{B_k}$. $Y$ has the same expectation as $X_1$, conditioned on the value of the hidden variable $H$, since \[\E(X_1\given H=j) =  (\mm_1)_j = \sum_{\ell=1}^k u_{1\ell} \E(X_{B_\ell} \given H=j)
 = \E(Y\given H=j).\] Moreover, given the value of $H$, $X_1$ and $Y$ are independent, 
because $\{1\}\cap B_\ell=\emptyset$.

\paragraph{Bootstrapping to obtain the second moment.} As a consequence of~\eqref{eq:u1B}, we can perform the bootstrapping step which is at the heart of our algorithm. We are interested in obtaining the vector 
\[v_2 \coloneqq \mm_1^{\odot 2} \pi_{\odot} \AA^\tpose \]
which is a linear image (under the mapping $\pi_{\odot} \AA^\tpose$) of a random variable distributed as the product of two random variables which are independent conditional on $H$; each distributed as $X_1$ conditional on $H$. This we get by setting
 \[v_2 \coloneqq (\E(X_1YX_{A_1}),\dotsc,\E(X_1YX_{A_k})), \] Since $A_1 = \emptyset$, the entry $(v_2)_1 = \E(X_1Y) = \mm_1^{\odot 2} \pi_{\odot} \One^\tpose$ is exactly our desired second moment. To get access to $v_2$, we observe that even though we don't necessarily have two independent copies of $\mm_1$ among our rows, our synthetic bit $Y$ provides the needed independence from $X_1$. 
As a matter of notation, for our collection $\calB \subseteq 2^T$, we define $\calB+\{1\}$ to consist of the sets $B_\ell \cup \{1\}$ for each $B_\ell \in \calB$.

 Now we can write 
\begin{equation}
 v_2 = \mm_1^{\odot 2}\pi_\odot \AA^\tpose =
(\mm_1 \odot (u_1\BB)) \pi_\odot \AA^\tpose 
= u_1 \CC_{\calB+\{1\},\calA}
 \label{boot2} \end{equation}
 Since we already have $u_1$, and since $\CC_{\calB+\{1\},\calA}$ is observable,  expression~\eqref{boot2} can be used to compute $v_2$. 

\paragraph{Bootstrapping to all moments.} The generalization of the second moment computation is this. Given vector $u_{r-1}$ defined by $u_{r-1}\coloneqq \mm_1^{\odot(r-1)}\BB^{-1}$, we define $v_r$ and $u_r$ and show how to compute them:
\begin{align} & \text{Definition} && && \text{Computation} \nonumber \\
 v_r &\coloneqq \mm_1^{\odot r}\pi_\odot \AA^\tpose &&\quad
= (\mm_1 \odot (u_{r-1}\BB)) \pi_\odot \AA^\top  
&& = u_{r-1} \CC_{\calB+\{1\},\calA}  \label{bootrperfectv}
\\
u_r &\coloneqq \mm_1^{\odot r}\BB^{-1} &&\quad= \mm_1^{\odot r} \pi_{\odot} \AA^\tpose {\AA^{\tpose}}^{-1} \pi_{\odot}^{-1} \BB^{-1} 
&& = v_r \CC_{\calB\calA}^{-1}.
 \label{bootrperfectu} \end{align}
 In the actual algorithm, we'll be working with empirical approximations of $v_r, u_r$, and $\CBA$, which we'll denote $\ev_r$, $\eu_r$, and $\eCBA$, respectively. A key part of the technical work will be in bounding error amplification.
 In order to compute $\ev_r$ and $\eu_r$, we'll use (following~\eqref{bootrperfectv},\eqref{bootrperfectu}) 
 the following assignments, where all quantities are empirical estimates:
 \begin{align}
 \ev_r & \coloneqq
\eu_{r-1} \eCC_{\calB+\{1\},\calA}
 \label{bootr1} \\
\eu_r & := \ev_r (\eCBA)^{-1} \label{bootr2} 
\end{align}

It is important to note that the bootstrapping can be performed only because $\eu_{r-1}$ places weight solely upon rows in $\calB \subseteq 2^T$.
Then the bootstrapping yields an expression for $\ev_r$ that contains moments involving $X_1$ as well as the entries in $T$. But once we compute $\eu_r$, we regain a coefficient vector
 for a synthetic bit,
 that again places weight solely upon rows in $\calB$.

The fact that both the computation of $\ev_r$ (in~\eqref{bootr1}) and of $\eu_r$ 
(in~\eqref{bootr2}) work stably and not only in the perfect-statistics limit, relies upon the following:
\begin{corollary} If all empirical multilinear moments are within $\eps < \zeta^{\Omega(k^2)}$ of their true values, then
$\Norm{\eCBA - \CBA} \leq k\eps$ and $\Norm{\eCBA^{-1} - \eCBA^{-1}} \leq \zeta^{-O(k^2)}\pi_{\min}^{-2}\eps$. 
\label{stab} \end{corollary}
\begin{proof}
This will be an immediate consequence of Lemma~\ref{lem:operator norms for C}.
\end{proof}

From this, we can bound the increase in error due to each subsequent application of \eqref{bootr1} and \eqref{bootr2}.
\begin{lemma} If all empirical multilinear moments are within $\eps < \zeta^{\Omega(k^2)}$ of their true values, then for all $i$, $\Normi{\eu_i - u_i} < \zeta^{-O(k^2)}\pi_{\min}^{-2}\Normi{\eu_{i-1} - u_{i-1}}$ and $\Norm{\ev_r-v_r}\leq \zeta^{-O(k^2)}\pi_{\min}^{-2}\Normi{\ev_{i-1} - v_{i-1}}$.\end{lemma}
\begin{proof}
This follows from Lemma~\ref{lem:bounds on u and v err}, which is itself a consequence of Lemma~\ref{lem:operator norms for C}.
\end{proof}

Nevertheless, if we compute each $\ev_r$ and $\eu_r$ as described above,
we would need to start with accuracies $\ep < \zeta^{\Omega(k^3)}$ in order to retain accuracy after bootstrapping $O(k)$ times, as is required by the above algorithm. 
To avoid this, we'll need to reduce the number of iterations by using $3k-3$, rather than $2k-1$, $\zeta$-separated rows. 

\subsubsection{Improved error control: constructing higher moments using $3k-3$ \, $\zeta$-separated rows} \label{sec:3k-3}
In order to avoid performing $k$ iterations to compute $\ev_k$ and $\eu_k$, we'll use another set of rows, $\calB' = \Set{B'_1,\dotsc,B'_k} \subseteq 2^{T'}$ of size $k$, where $T'$ is disjoint from $S$ and $T$, and $B'_1 = \emptyset$. Now we'll introduce $\BB' \coloneqq \MM[\BB']$ and $\CBAp \coloneqq \BB' \pi_{\odot} \AA^\tpose$, both of which will be invertible as before. 

Previously, $u_i$ was a linear combination of the rows of $\BB$ such that $u_i\BB = \mm_1^{\odot i}$. We'll introduce a new, but similar, sequence of vectors $u'_i$ where $u'_i\BB' = \mm_1^{\odot i}$ and $u'_i$ is obtained from $v_i$ by $u'_i = v_i\CBAp^{-1}$. In the algorithm, we'll only have access to the approximations $\eu'_i$ and $\eCBAp$ and we'll compute each successive $\eu'_i$ by $\eu'_i = \ev_i\eCBAp^{-1}$. 

The advantage we obtain over the more straightforward process of the previous section, results from our ability to compute $v_{2^i}$ using $u_{2^{i-1}}$ and $u'_{2^{i-1}}$; in that way, we are able to get away with performing only $1+\lg k$ iterations to compute any of $v_1,\dotsc, v_{2k}$.

To describe the computation, we first define the sum of two collections of subsets $\mathcal{U},\mathcal{V}\subseteq 2^{[n]}$ by \[\mathcal{U} + \mathcal{V} \coloneqq \Set{U\cup V: U\in \mathcal{U}, V\in \mathcal{V}}.\] Now let $x$ and $y$ be vectors indexed by the the subsets in $\calB$ and $\calB'$, respectively. Recall that $\calB \subseteq 2^T, \calB' \subseteq 2^{T'}$, where $T\cap T' = \emptyset$. Then $\Card{\calB + \calB'} = k^2$ and each subset in $\calB + \calB'$ can be uniquely written as $B_\ell \cup B'_j$ for $\ell,j \in [k]$. We define the Kronecker product $(x \otimes y) \in \R^{\calB + \calB'}$ to be the vector indexed by subsets in $\calB + \calB'$ given by 
\[ (x \otimes y)_{B_\ell \cup B'_j} \coloneqq x_{B_\ell}y_{B'_{j}} \] for any $\ell,j \in [k]$. To access $v_{2^i}$, we write 
\[ v_{2^i} = \mm_1^{\odot 2^i} \pi_{\odot} \AA^\tpose = ((u_{2^{i-1}}\BB)\odot (u'_{2^{i-1}}\BB)) \pi_{\odot} \AA^\tpose = (u_{2^{i-1}} \otimes u'_{2^{i-1}}) \CC_{\calB + \calB', \calA}, \] expressing the row $v_{2^i}$ as the linear of combination of the $k^2$ rows corresponding to the subsets in $\calB + \calB'$. 

Of course, we'll need to compute $v_\ell$ for $\ell$ not a power of $2$. We can do this using a slight modification of the recursive procedure where for $i=1,\dotsc,\lg k + 1$ and $j=1,\dotsc,2^i$, we compute
\begin{equation}
\ev_{\ell} \coloneqq (\eu_{j} \otimes \eu'_{2^i})\CC_{\calB + \calB', \calA}	
\end{equation}
 where $\ell = j + 2^i$, and we've computed $\eu_j$ in a prior iteration for all $j \leq 2^i$. 

Under this modification, each $\ev_i$ is produced in only $\log k$ iterations each of which involves a matrix multiplication by $\eCBA^{-1}$ or a convolution followed by multiplication by $\CC_{\calB + \calB', \calA}$, each step of which can increase the error in total by $\zeta^{-O(k^2\log k)}\pi_{\min}^{-2}$. By starting with empirical moments accurate to within $\zeta^{\Omega(k^2\log k)}\pi_{\min}^{2k}$, we can ensure that the resulting vectors $\ev_i$ and $\eu_i$ are sufficiently close to start solving for $\mm_1$ and $\pi$. 

We provide pseudocode for this ``$3k-3$ rows'' version of the algorithm, see Fig.~\ref{alg}. 

\paragraph{Flagging a failure condition} If the chosen rows $S \cup T \cup T'$ fail to all be $\zeta$-separated, the algorithm might fail. However, we will detect such failure. The conditions that we actually need so that the algorithm should work, are these: (a) $\eCBA$ and
$\eCC_{\calB',\calA}$ should have a large least singular value. (It does not actually matter whether all rows we use are $\zeta$-separated, that was merely a sufficient condition for this well-conditioning.) We compute this singular value explicitly and simply dismiss the triple $S,T,T'$ if this condition fails. (b) The first row of $\calB$, namely $\{1\}$ in the numbering used in the pseudocode for Algorithm~\ref{alg}, should be $\zeta$-separated. If condition (a) holds but this condition fails, we will detect the failure in line~\ref{hankel-line} of the algorithm, because the Hankel matrix will have insufficient eigenvalue gap (see Cor.~12 of~\cite{GMRS20}). 

\subsection{Solving the power distribution problem}
Once we've computed $\ev_1,\dotsc, \ev_{2k}$, we have access to all of the moments 
of the $k$-spike distribution corresponding to observable $X_1$. (The $i$th moment is the first entry of the corresponding vector $\ev_i$). Recall, these are the moments of a  mixture of $k$ Bernoulli random variables, where the $r$'th moment corresponds to drawing a mixture component $j$ with probability $\pi_j$, then setting the Bernoulli random variable to $1$ with probability $\mm_{1j}^r$. 
The problem of recovering the parameters (i.e., the vectors $\mm_1$ and $\pi$) from approximate moments of this form has been extensively studied, and many algorithms have been provided. We use the algorithm $\LearnPowerDistribution$ from \cite{GMRS20}, which on inputs accurate to within $\eps$, outputs parameters $\tilde{\mm}_1$ and $\tilde{\pi}$ to within accuracy $\frac{1}{\pi_{\min}} \zeta^{-O(k)} \eps$
 while running in time (arithmetic operations) $k^{2+o(1)}$. 

\subsection{Recovering the remaining parameters}
Once we have estimates for $\mm_1$ and $\pi$, we can simply solve for $\AA$, $\BB$, and $\BB'$ using the fact that $\AA^\tpose = \pi_{\odot}^{-1}(\Vandermonde(\mm_1))^{-1}V$ where 
\[ \tilde{V} = (v_0;\dotso ; v_{k-1}) \] is the matrix with rows $v_i$ and $\Vandermonde(\mm_1)$ is the Vandermonde matrix with rows $\mm_1^{\odot i}$ for $i=0,\dotsc, k-1$. Note that here we finally do require that $\mm_1$ be $\zeta$-separated. We can thus solve for $\AA^\tpose$. 

To solve for $\BB$ we use $\BB = \CBA (\AA^\tpose)^{-1}\pi_{\odot}^{-1}$. Likewise for $\BB'$
we use $\BB' = \CBAp (\AA^\tpose)^{-1}\pi_{\odot}^{-1}$.

Now for any row $i$ not already computed, we need only pick any other set $\mathcal{S}=\Set{S_1,\dotsc,S_k}$ of $k$ linearly independent rows supported on a set not containing $i$, and we can solve for $\mm_i$ by writing \[\mm_i = (\E[X_iX_{S_1}], \dotsc, \E[X_iX_{S_k}]){\MM[\mathcal{S}]^\tpose}^{-1}\pi_{\odot}^{-1}\] In particular, by setting $\mathcal{S} = \calA$ we can solve for all rows in $[n]\setminus S$ and by setting $\mathcal{S} = \calB$ we can solve for all rows in $[n]\setminus T$. Together, this suffices to solve for all rows.

\suppress{\subsection{Controlling error}
The algorithm requires inverting $\tilde{C}_{\mathcal{B'A}}$ and $\tilde{C}_{\mathcal{BA}}$, so the error will depend on these matrices being well conditioned. We will show in Section~\ref{sec:condition} that a well conditioned $\C_{\mathcal{B'A}}$ and $C_{\mathcal{BA}}$ can be found under the assumption of enough sources with well separated observables. We do this by first showing $\MM[2^S]$ is well conditioned and then showing that this implies $\AA$ is also well conditioned. The same arguments exist for $\BB$ and $\BB'$. When this is combined with lower bounded mixture weights $\pi$, we get well conditioned $C_{\mathcal{B'A}}$ and $C_{\mathcal{BA}}$, and thus well conditioned empirical $\tilde{C}_{\mathcal{B'A}}$ and $\tilde{C}_{\mathcal{BA}}$ provided the input statistics are within the required accuracy.}

\subsection{Runtime}
The algorithm contains three main parts:
\begin{enumerate}
\item Find disjoint $S, T, T' \subset [n]$ and $\mathcal{A} \subset 2^S$, $\mathcal{B} \subset 2^T$, and $\mathcal{B} ' \subset 2^{T'}$, which is complexity $n^{O(k)} 2^{O(k^2)}$. First we require $n^{O(k)}$ iterations to check all possible disjoint $S, T, T'$. Then, $2^{O(k^2)}$ operations are required in each iteration to check all size $k$ subsets of the $2^{k-1}$ rows of $\MM[2^S], \MM[2^T]$ and $\MM[2^{T'}]$.
\item Nested loops to compute higher order moments $\tilde{v}$, and the corresponding $\tilde{u}$. This step takes time $O(\text{poly}(k))$.

\item Applying the power distribution result. This can be done in time $O(k^{2 + o(1)} + k (\log^2 k)\log \log(\eps^{-1}))$ (see Corollary \ref{cor:applypowerdistribution}).
\end{enumerate}
This gives runtime complexity of $n^{O(k)} 2^{O(k^2)}  + O(k^{2 + o(1)} + k (\log^2k) \log\log(\eps^{-1}))$.
If all sources are $\zeta$-separated, we do not need to iterate over choices of $S, T, T'$, so the runtime improves to $2^{O(k^2)}  + O(k^{2 + o(1)} + k \log^2(k) \log\log (\eps^{-1}))$.

\begin{algorithm}[t]
\caption{Identifies a mixture of product distributions given $3k-3$ \, $\zeta$-separated observable bits}\label{alg}	
	\begin{algorithmic}[1]
		\State Let 
		 $\calB\subseteq 2^{\Set{1,\dotsc, k-1}}, \calB' \subseteq 2^{\Set{k,\dotsc, 2k-2}},\calA \subseteq 2^{\Set{2k-1,\dotsc,3k-3}}$,
		 with $\Card{\calA} = \Card{\calB} = \Card{\calB'} = k$ maximizing 
		 $\min\{ \sigma_k(\eCC_{\calB\calA}), \sigma_k(\eCC_{\calB'\calA}) \}$. If this min is below $\pi_{\min}\zeta^{O(k^2)}$, terminate.
		 Denote $\calB = \Set{B_1,\dotsc, B_k}$, $\calB' = \Set{B'_1,\dotsc,B'_k}$, and $\calA = \Set{A_1,\dotsc, A_k}$. Without loss of generality $B_1=\{1\}$. 
		\State $\ev_0 \gets (\emom(A_1);\dotso; \emom(A_k))$. 
		\State $\ev_1 \gets (\emom(A_1 \cup \Set{1}); \dotso; \emom(A_k \cup \Set{1}))$.
		\State $\eu_1 \gets \ev_1 (\eCBA)^{-1}$. 
		\State $\eu'_1 \gets \ev_1 (\eCC_{\calB'\calA})^{-1}$.
		\For{$i=2,\dotsc,\log k+1$}
			\For{$j=1,\dotsc,2^{i-1}$}
				\State $\ev_{2^{i-1}+j}\gets (\eu_j \otimes\eu'_{2^{i-1}}) \eCC_{\calB + \calB', \calA}$
				\State $\eu_{2^{i-1}+j} \gets \ev_{2^{i-1}+j}(\eCC_{\calB\calA})^{-1}$.
			\EndFor
			\State $\eu'_{2^i} \gets \ev_{2^i}(\eCC_{\calB'\calA})^{-1}$.
		\EndFor

		\State Let $H_{k+1}$ be the $(k+1)\times(k+1)$ Hankel matrix with entries given by 
			$[H_{k+1}]_{i,j=0}^k = (\ev_{i+j})_1$. If the second-smallest eigenvalue of $H_{k+1}$ is below $\frac{\pi_{\min}}{2} (\zeta/16)^{2k-2}$, terminate.
			\label{hankel-line} 
		\State $\tilde{\mm}_1, \tilde{\pi} \gets \LearnPowerDistribution(H_{k+1})$. 
		\State $\tilde{V} \gets (\ev_0;\dotso;\ev_{k-1})$. 
		\State $\Vandermonde(\tilde{\mm}_1) \gets (\tilde{\mm}_1^{\odot 0};\dotso;\tilde{\mm}_1^{\odot (k-1)})$.
		\State $\tilde{\AA}^{\tpose} \gets \tilde{\pi}_{\odot}^{-1}(\Vandermonde(\tilde{\mm}_1))^{-1}\tilde{V}$.
		\State $\tilde{\BB} \gets \eCBA(\tilde{\AA}^\tpose)^{-1}\pi_{\odot}^{-1}$.
		\State $\tilde{\BB'} \gets \eCBAp(\tilde{\AA}^\tpose)^{-1}\pi_{\odot}^{-1}$. 
		\State For every $i\in [n]\setminus [k]$, $\tilde{\mm}_i \gets \Paren{\emom(A_1\cup \Set{i}), \dotsc, \emom(A_k\cup \Set{i})}^{\tpose} {\Paren{\tilde{\AA}^{\tpose}}}^{-1} \tilde{\pi}_{\odot}^{-1}$. 
		\State For every $i \in \Set{2,\dotsc,k}$, $\tilde{\mm_i} \gets \Paren{\emom(B_1 \cup \Set{i}), \dotsc, \emom(B_k\cup \Set{i})}^{\tpose} {\Paren{\tilde{\BB}^{\tpose}}}^{-1} \tilde{\pi}_{\odot}^{-1}$.
\end{algorithmic}
\end{algorithm}

\section{Analyzing the Algorithm}
As we have seen the algorithm consists of bootstrapping steps which, as indicated in Eqns.~\eqref{bootr1},~\ref{bootr2}, lift us ``forward'' from $\eu_{r-1}$ to $\ev_r$ and then ``back'' to $\eu_r$. We must now control the loss in accuracy of the statistics, in each of these steps. It turns out that the first of these is easier and less expensive in accuracy; while the second, in which we ``invert'' from approximate statistics to obtain a linear combination of sources, is harder and also more expensive. In this section we show how to control these steps. We rely for this control on a condition number bound (which is not in itself algorithmic and is due entirely to the $\zeta$-separation), which will be given in Section~\ref{sec:condition}.
Throughout the analysis, we'll assume that every multi-linear moment we use is known with additive error bounded by 
\[  \eps \coloneqq \zeta^{C_1k^2\log k}\pi_{\min}^{C_2\log k} \] for constants $C_1,C_2$. Choosing $C_1 \coloneqq 60, C_2 \coloneqq 8$ is sufficient to give us final error $\zeta^{\Omega(k^2\log k)}$. 
\subsection{Bounding $\Norm{\eu_{j}-u_{j}}$ for $j \leq 2k$} 
\begin{definition} We define
		$\beta \coloneqq \frac{(\zeta/2)^{k-1}}{3k^3}$. We'll frequently use the bound $\beta \geq \zeta^{3k}$.  
\end{definition}
The following Lemma is a consequence of Theorem~\ref{core-stab} to be proven in the next section:
\begin{lemma}
When the input mixture contains $3k-3$ $\zeta$-separated rows can find disjoint sets $S, T, T' \subseteq [n]$ of size $k-1$ each and subsets $\calA \subseteq 2^S,\calB \subseteq 2^T, \calB' \subseteq 2^{T'}$ with $\Card{\calA},\Card{\calB} = k$ such that the matrices $\AA \coloneqq \MM[\calA]$, $\BB \coloneqq \MM[\calB]$, and $\BB' \coloneqq \MM[\calB']$ satisfy 
\begin{enumerate}
\item The first row of $\AA$, $\BB$, and $\BB'$ is the all-ones vector, $\One$. 
\item $\sigma_k(\AA),\sigma_k(\BB),\sigma_k(\BB') \geq \beta^k2^{-3k/2}k^{-3/2}$, and 
\item $\sigma_{\max}(\MM[\calA]), \sigma_{\max}(\MM[\calB]), \sigma_{\max}(\MM[\calB]) \leq k$.
\end{enumerate}
Moreover, the first row of each of these matrices is $\One$. 
Finally, the derived matrices $\CBA = \BB \pi_{\odot} \AA^\tpose$ and $\CBAp = \BB' \pi_{\odot} \AA^\tpose$ satisfy 
\begin{enumerate}
\item $\sigma_{\max}(\CBA), \sigma_{\max}(\CBAp) \leq k^2$, and 
\item $\sigma_{k}(\CBA), \sigma_{k}(\CBAp) \geq \beta^{2k}2^{-3k}k^{-3}\pi_{\min}$.
\end{enumerate}
 \label{lem:M subsets singular values}
\end{lemma}
\begin{proof} This follows immediately from Theorem~\ref{core-stab}, the definition  $\CBA = \MM[\calB]\pi_{\odot}\MM[\calA]^{\tpose}$, and the min-max characterization of the first and last singular values. 	
\end{proof}

\begin{corollary}
$\Norm{(\CBA)^{-1}} \leq \zeta^{-10k^2}\pi_{\min}^{-1}$. \label{cor:CBA inv norm}
\end{corollary}
\begin{proof}
$\Norm{(\CBA)^{-1}} \leq (\zeta^{-3k})^{2k} 2^{3k} k^3 \pi_{\min}^{-1} \leq \zeta^{-10k^2}\pi_{\min}^{-1}$
\end{proof}

\begin{corollary}
$\Norm{\AA^{-1}}, \Norm{\BB^{-1}}, \Norm{(\BB')^{-1}} \leq \zeta^{-6k^2}$. \label{cor:bound ABB' inv}
\end{corollary}

\begin{lemma}
$\Norm{u_i} \leq \zeta^{-6k^2}$, and $\Norm{v_i} \leq \zeta^{-1}$.\label{lem:bound ui and vi}
\end{lemma}
\begin{proof}
We observe that $\Norm{u_i} = \Norm{\mm_1^{\odot i}\BB^{-1}} \leq k\Norm{\BB^{-1}} \leq \beta^{-k}2^{3k/2}k^{3/2}\leq \zeta^{-6k^2}$. On the other hand, $\Norm{v_i} \leq k \leq \zeta^{-1}$, since $v_i$ is a vector of moments of products of Bernoulli random variables.
\end{proof}
\begin{observation}
$\Norm{\CBA},\Norm{\CBBA} \leq k^3$. If all moments are within $\eps$ of their true values, $\Norm{\eCBA}, \Norm{\eCBBA}\leq 2k^3$.
\end{observation}

\begin{lemma}
If all multilinear moments are within $\eps$ of their true values, then 
\[ \Norm{\eCBA-\CBA}_{2} \leq \zeta^{-1}\eps,\qquad \Norm{\eCC_{\calB + \calB',\calA} - \CC_{\calB + \calB', \calA}} \leq \zeta^{-2}\eps, \]and 
\[ \Norm{\eCBA^{-1}-\CBA^{-1}} \leq \zeta^{-26k^2}\pi_{\min}^{-2}\eps.\] \label{lem:operator norms for C}
\end{lemma}
\begin{proof}
The first two inequalities just use $\Norm{\cdot}_{2} \leq \Norm{\cdot}_F$. For the final inequality we use Lemma~\ref{lem:op norm for inv and perturb inv}:
\begin{align*} 
\Norm{\eCBA^{-1}-\CBA^{-1}} \leq 2\Norm{\CBA^{-1}}^2\Norm{\CBA-\eCBA} \leq \beta^{-4k}2^{6k+1}k^7\pi_{\min}^{-2}\eps
\end{align*}

\end{proof}

\begin{lemma} When the assumptions of Lemma~\ref{lem:operator norms for C} are satisfied, we have for any $i\in [2k]$ and $j = \Ceil{\log i}$, 
\[\Norm{\ev_i - v_i},\Norm{\eu_{j} - u_{j}},\Norm{\eu'_{j}-u'_{j}} \leq \zeta^{-42ik^2}\pi_{\min}^{-2i}\eps.\] \label{lem:bounds on u and v err}

\end{lemma}
\begin{proof} 
Recall that we initialize the algorithm with 
\[ \ev_1 \gets (\emom(A_1 \cup \Set{1}),\ \dotsc, \emom(A_k \cup \Set{1})),\qquad \eu_1 \gets \ev_1 (\eCBA)^{-1},\qquad \eu'_1 \gets \ev_1 (\eCC_{\calB'\calA})^{-1}. \]
First, we observe that $\Norm{\ev_1 - v_1} \leq \eps$ by assumption. Since $\eu_1,\eu'_1$ are computed in the same manner here as in the loop, we'll bound that error in the induction. Now assume that the claim holds up to $i-1$. Recall that in each iteration of the outer loop we compute
\[\ev_{2^i} \gets (\eu_{2^{i-1}} \otimes \eu'_{2^{i-1}})\eCBBA,\qquad \eu_{2^i} \gets \ev_{2^i}(\eCC_{\calB\calA})^{-1},\qquad \eu'_{2^i} \gets \ev_{2^i}\eCC_{\calB'\calA}.\] We'll first focus on bounding $\Normi{\ev_{2^i} - v_{2^i}}$. To do this we write 
\begin{align*}
\ev_{2^i}-v_{2^i}
&= (\eu_{2^{i-1}} \otimes \eu'_{2^{i-1}})\eCBBA - (u_{2^{i-1}} \otimes u'_{2^{i-1}})\CBBA
\end{align*} and letting $w = \eu_{2^{i-1}}-u_{2^{i-1}}$, $w' = \eu'_{2^{i-1}}-u_{2^{i-1}}$, and $E = \eCBBA - \CBBA$ we can bound the norm of the difference as follows, using the bilinearity of the Kronecker product, Lemma~\ref{lem:bound ui and vi}, and the induction hypothesis: 
\begin{align*}
\Norm{\ev_{2^i}-v_{2^i}}
&= \Norm{(\eu_{2^{i-1}} \otimes \eu'_{2^{i-1}})\eCBBA - (u_{2^{i-1}} \otimes u'_{2^{i-1}})\CBBA}\\
&\leq \Norm{(w\otimes \eu'_{2^{i-1}})\eCBBA} + \Norm{(\eu_{2^{i-1}}\otimes w')\eCBBA} + \Norm{(\eu_{2^{i-1}} \otimes \eu'_{2^{i-1}})E}\\
&\leq 2\zeta^{- 42(i-1)k^2}\pi_{\min}^{-2(i-1)}\zeta^{-6k^2}k^3\eps + \zeta^{-12k^2}\zeta^{-2}\eps\\
&\leq \zeta^{-(42(i-1)+16)k^2}\pi_{\min}^{-2(i-1)}\eps \end{align*}

Now we can bound $\Normi{\eu_{2^i}-u_{2^i}}$ by observing that  
\begin{align*} 
\eu_{2^i} - u_{2^i}
&= \ev_{2^i}\eCBA^{-1} - v_{2^i}\CBA^{-1}.
\end{align*} Let $z = \ev_{2^i} - v_{2^i}$ and let $D = \eCBA^{-1} - \CBA^{-1}$. The above equation becomes 
\begin{align*} 
\eu_{2^i} - u_{2^i}
= (v_{2^i}+z)(\CBA^{-1} + D) - v_{2^i}\CBA^{-1} = v_{2^i}\CBA^{-1} + z\CBA^{-1} + zD 
\end{align*} and after taking norms and using the triangle inequality we obtain 
\begin{align*}
\Norm{\eu_{2^i} - u_{2^i}}
&\leq \Norm{v_{2^i}D} + \Norm{z\CBA^{-1}} + \Norm{zD} 
\end{align*} By Corollary~\ref{cor:CBA inv norm}, Lemma~\ref{lem:bound ui and vi} and the induction hypothesis, we get 
\begin{align*}
\Norm{\eu_{2^i} - u_{2^i}}
&\leq \zeta^{-1}\zeta^{-26k^2}\pi_{\min}^{-2}\eps + \zeta^{-(42(i-1)+16)k^2}\pi_{\min}^{-2(i-1)}\eps \zeta^{-16k^2}\pi_{\min}^{-1} + \zeta^{-(42(i-1)+16)k^2}\pi_{\min}^{-2(i-1)}\zeta^{-26k^2}\pi_{\min}^{-2}\eps\\
&\leq \zeta^{-26k^2-1}\pi_{\min}^{-2}\eps + \zeta^{-(42(i-1)+20)k^2}\pi_{\min}^{-2(i-1)-1}\eps + \zeta^{-(42(i-1)+42)k^2}\pi_{\min}^{-2(i-1)}\pi_{\min}^{-2}\eps\\
&\leq \zeta^{-42i k^2}\pi_{\min}^{-2i} \eps
\end{align*}
For $j$ not a power of $2$, we can do the same analysis, and since the error bound is increasing in $j$, the result will follow.
\end{proof}
\begin{corollary} Algorithm~\ref{alg} will produce vectors $\ev_i$ for $i \leq 2k$ satisfying 
\[\Norm{\ev_i - v_i} \leq \zeta^{-42k^2(\log k+1)}\pi_{\min}^{-2(\log k + 1)}\eps.\]
\end{corollary}

\subsection{Applying the power distribution result}
\begin{definition}Given a mixture $\mathcal{M}$ of $k$ Bernoulli random variables with probabilities $m_1,\dotsc,m_k$ and mixing probabilities $\pi_1,\dotsc, \pi_k$, respectively, let $[\mathcal{H}_{k+1}]_{i,j=0}^k = \mu_{i+j}$ be the matrix of moments of the distribution.	
\end{definition}
\begin{theorem}[Theorem~17 from \cite{GMRS20}]
Given a mixture $\mathcal{M} = (m,\pi)$ as above where $m$ is $\zeta$-separated, there is an algorithm, $\LearnPowerDistribution$, that takes a Hankel matrix $[\tildecal{H}_{k+1}]_{i,j=0}^k = \tilde{\mu}_{i+j}$ of approximate moments of $\mathcal{M}$ satisfying $\Norm{\tildecal{H}_{k+1} - \mathcal{H}_{k+1}}_2 \leq \pi_{\min}2^{-\gamma}\zeta^{16k}$ (for some $\gamma \geq 1$) and outputs a model $\tildecal{M} = (\tilde{m},\tilde{\pi})$ satisfying \[ \Normi{\tilde{m}-m},\Normi{\tilde{\pi}-\pi}\leq 2^{-\gamma} \] using $O(k^2\log k + k\log^2 k\cdot 
\log(\log\zeta^{-1} + \log \pi_{\min}^{-1} + \gamma))$ arithmetic operations.
\end{theorem}

\begin{corollary}
\label{cor:applypowerdistribution}
The output $(\tilde{\mm}_1,\tilde{\pi})$ of $\LearnPowerDistribution$ in line 14 of Algorithm~\ref{alg} will satisfy \[\Norm{\tilde{\mm}_1-\mm_1},\Norm{\tilde{\pi}-\pi} \leq \zeta^{-42k^2(\log k + 1)-16k-1}\pi_{\min}^{-2\log k -3}\eps. \] This step will use 
$O(k^2\log k + k\log^2 k\cdot 
\log\log(\eps^{-1}))$ arithmetic operations.  
\end{corollary}
\begin{proof} Every entry $(\ev_i)_1$ satisfies $\Norm{(\ev_i)_1 - (v_i)_1} \leq \zeta^{-42k^2(\log k+1)}\pi_{\min}^{-2(\log k+1)} \eps$ so \[\Norm{\tildecal{H}_{k+1}-\mathcal{H}_{k+1}} \leq \zeta^{-42k^2(\log k+1)}\pi_{\min}^{-2(\log k+1)}\] which implies that 
\[\Normi{\tilde{\mm}_1 - \mm_1}, \Normi{\tilde{\pi} - \pi} \leq \zeta^{-42k^2(\log k + 1)-16k}\pi_{\min}^{-2(\log k + 1) - 1}\eps.\] Finally, we add a factor of $\zeta^{-1}$ to convert back to the Euclidean norm to get the stated bound.	
\end{proof}

\subsection{Solving for the rest of the model}
Once we've computed $\tilde{\mm}_1$ and $\tilde{\pi}$, we'll use them to compute the remaining model parameters. In this section we bound the additional error introduced by these computations. 

\begin{observation}
$\Norm{\Vandermonde(\tilde{\mm}_1) - \Vandermonde(\mm_1)} \leq \zeta^{-1}\Norm{\tilde{\mm}_1 - \mm_1}$.
\end{observation}

\begin{observation}[Claim 26 in \cite{GMRS20}]
$\Norm{\Vandermonde(\mm_1)^{-1}} \leq 2^k/\zeta^{k-1} \leq \zeta^{-2k}$ when $\mm_1$ is $\zeta$-separated. \label{obs:vdm condition}
\end{observation}

\begin{lemma}
The computed $\tilde{\AA}$ produced by Algorithm~\ref{alg} will satisfy
$\Norm{\tilde{\AA} - \AA} \leq \zeta^{-42k^2(\log k + 1)-20k-6}\pi_{\min}^{-2\log k -5}\eps$. 
\end{lemma}
\begin{proof}
First, we observe that $\Norm{\tilde{V}} \leq \Norm{V} \leq \zeta^{-2}$ and $\Norm{\Vandermonde(\tilde{\mm}_1)^{-1}} \leq \zeta^{-3k}$ by Lemma~\ref{lem:op norm for inv and perturb inv} and Observation~\ref{obs:vdm condition}. Now $\Norm{\tilde{\pi}_{\odot}^{-1}-\pi_{\odot}^{-1}} \leq \zeta^{-42k^2(\log k + 1)-16k-2}\pi_{\min}^{-2\log k -5}\eps$ by Lemma~\ref{lem:op norm for inv and perturb inv}.
Thus, \begin{align*}
 \Norm{\tilde{\pi}_{\odot}^{-1}-\pi_{\odot}^{-1}}\Norm{(\Vandermonde(\tilde{\mm}_1))^{-1}}\Norm{\tilde{V}} \leq \zeta^{-42k^2(\log k + 1)-18k-5}\pi_{\min}^{-2\log k -5}\eps.
\end{align*}
Now $\Norm{(\Vandermonde(\tilde{\mm}_1))^{-1}-(\Vandermonde(\mm_1))^{-1}} \leq \zeta^{-42k^2(\log k + 1)-20k-2}\pi_{\min}^{-2\log k -3}\eps$ by Lemma~\ref{lem:op norm for inv and perturb inv}, so 
\[ \Norm{\tilde{\pi}_{\odot}^{-1}}\Norm{(\Vandermonde(\tilde{\mm}_1))^{-1}-(\Vandermonde(\mm_1))^{-1}}\Norm{\tilde{V}} \leq \zeta^{-42k^2(\log k + 1)-20k-4}\pi_{\min}^{-2\log k -4}\eps. \] Finally, $\Norm{\tilde{V} - V} \leq \zeta^{-42k^2(\log k+1)-1}\pi_{\min}^{-2\log k + 1}\eps$ so that 
\[ \Norm{\pi_{\odot}^{-1}}\Norm{(\Vandermonde(\tilde{\mm}_1))^{-1}}\Norm{\tilde{V} - V} \leq \zeta^{-42k^2(\log k+1)-3k-1}\pi_{\min}^{-2\log k -3}\eps. \]
Putting these together, we easily obtain
\begin{align*}
\Norm{\tilde{\AA} - \AA} 
&= \Norm{\tilde{\pi}_{\odot}^{-1}(\Vandermonde(\tilde{\mm}_1))^{-1}\tilde{V}	- \pi_{\odot}^{-1}(\Vandermonde(\mm_1))^{-1}V}\\
&\leq \Norm{\tilde{\pi}_{\odot}^{-1}-\pi_{\odot}^{-1}}\Norm{(\Vandermonde(\tilde{\mm}_1))^{-1}}\Norm{\tilde{V}} + \Norm{\tilde{\pi}_{\odot}^{-1}}\Norm{(\Vandermonde(\tilde{\mm}_1))^{-1}-(\Vandermonde(\mm_1))^{-1}}\Norm{\tilde{V}}\\
&\qquad+ 
 \Norm{\pi_{\odot}^{-1}}\Norm{(\Vandermonde(\tilde{\mm}_1))^{-1}}\Norm{\tilde{V} - V}  
(\Norm{\Vandermonde(\mm_1)^{-1}E_2} + \Normi{E_1V}) + \Norm{w}\Norm{\Vandermonde(\tilde{\mm}_1)^{-1}}\Norm{\tilde{V}}\\
&\leq \zeta^{-42k^2(\log k + 1)-20k-6}\pi_{\min}^{-2\log k -5}\eps. \qedhere\end{align*}
\end{proof}

\begin{lemma} The matrices $\tilde{\AA},\tilde{\BB},$ and $\tilde{\BB}'$ satisfy 
\[ \Norm{\tilde{\AA}^{-1}}, \Norm{\tilde{\BB}^{-1}}, \Norm{(\tilde{\BB}')^{-1}} \leq \zeta^{-7k^2}. \]
\end{lemma}
\begin{proof} By Lemma~\ref{lem:op norm for inv and perturb inv} and Corollary~\ref{cor:bound ABB' inv}.
\end{proof}	

\begin{lemma}
The matrices $\tilde{\BB}'$ and $\tilde{\BB}$ produced by Algorithm~\ref{alg} will satisfy \[\Norm{\tilde{\BB}' - \BB'}, \Norm{\tilde{\BB}-\BB} \leq \zeta^{-42k^2(\log k + 1)-14k^2-20k-16}\pi_{\min}^{-2\log k -6}\eps.\] 
\end{lemma}
\begin{proof}
We'll prove the claim for $\tilde{\BB}$; the proof is identical for $\tilde{\BB}'$. We can bound $\tilde{\BB} - \BB$ using the same tools as in the previous bounds. First, we bound 
$\Norm{\eCBA-\CBA}\Norm{\tilde{\AA}^{-1}}\Norm{\tilde{\pi}_{\odot}^{-1}} \leq \zeta^{-7k^2-2}\pi_{\min}^{-1}\eps$. Now 
\[ \Norm{\eCBA}\Norm{\tilde{\AA}^{-1}-\AA^{-1}}\Norm{\tilde{\pi}^{-1}_{\odot}} \leq \zeta^{-4}\zeta^{-42k^2(\log k + 1)-14k^2-20k-11}\pi_{\min}^{-2\log k -6}\eps \] and 
\[ \Norm{\eCBA}\Norm{\tilde{\AA}^{-1}}\Norm{\tilde{\pi}_{\odot}^{-1}-\pi_{\odot}^{-1}} \zeta^{-4}\zeta^{-7k^2}\zeta^{-42k^2(\log k + 1)-7k^2-16k-6}\pi_{\min}^{-2\log k -5}\eps. \] The resulting bound is
\begin{align*}
\Norm{\tilde{\BB} -\BB}
&= \Norm{\eCBA(\tilde{\AA}^\tpose)^{-1}\pi_{\odot}^{-1}-\CBA(\AA^{\tpose})^{-1}\pi_{\odot}}\\
&\leq \Norm{\eCBA-\CBA}\Norm{\tilde{\AA}^{-1}}\Norm{\tilde{\pi}_{\odot}^{-1}} + \Norm{\eCBA}\Norm{\tilde{\AA}^{-1}-\AA^{-1}}\Norm{\tilde{\pi}^{-1}_{\odot}} + \Norm{\eCBA}\Norm{\tilde{\AA}^{-1}}\Norm{\tilde{\pi}_{\odot}^{-1}-\pi_{\odot}^{-1}}\\
&\leq\zeta^{-42k^2(\log k + 1)-14k^2-20k-16}\pi_{\min}^{-2\log k -6}\eps.
\end{align*}
\end{proof}

\begin{lemma}
Algorithm~\ref{alg} will compute $\tilde{\mm}_i$ satisfying $\Normi{\tilde{\mm}_i - \mm_i} \leq \zeta^{-42k^2(\log k + 1)-14k^2-20k-19}\pi_{\min}^{-2\log k -8}\eps$ for all $i$. 
\end{lemma}
\begin{proof}
We'll compute the bound using the inversion of $\tilde{\BB}$ since this will give us the worst case.  Let $\tilde{y} = (\emom(B_1\cup \Set{i}),\dotsc,\emom(B_k\cup\Set{1})$ and let $y = (\mom(B_1\cup \Set{i}),\dotsc,\mom(B_k\cup\Set{1})$. We note that $\Norm{\tilde{y} - y} \leq \zeta^{-1} \zeta^{-1}\eps$ by assumption, and $\Norm{\tilde{y}} \leq \zeta^{-2}$. Then 
\[  \Norm{\tilde{y}-y}\Norm{\tilde{\BB}^{-1}}\Norm{\tilde{\pi}_{\odot}^{-1}} \leq \zeta^{-7k^2-2}\pi_{\min}^{-2}\eps, \] 
\[ \Norm{\tilde{y}}\Norm{\tilde{\BB}^{-1}-\BB^{-1}}\Norm{\tilde{\pi}_{\odot}^{-1}} \leq \zeta^{-42k^2(\log k + 1)-14k^2-20k-19}\pi_{\min}^{-2\log k -8}\eps,\] and 
\[ \Norm{\tilde{y}}\Norm{\tilde{\BB}^{-1}}\Norm{\tilde{\pi}_{\odot}^{-1} - \pi_{\odot}^{-1}} \leq \zeta^{-42k^2(\log k + 1)-16k-7k^2-3}\pi_{\min}^{-2\log k -5}\eps^2. \] Of the three terms, the middle one clearly dominates so that we get 
\begin{align*}
\Norm{\tilde{\mm}_i - \mm_i}
&= \Norm{\tilde{y}\tilde{\BB}^{-1}\tilde{\pi}_{\odot}^{-1} - y\BB^{-1}\pi_{\odot}^{-1}}\\
&\leq \Norm{\tilde{y}-y}\Norm{\tilde{\BB}^{-1}}\Norm{\tilde{\pi}_{\odot}^{-1}} + \Norm{\tilde{y}}\Norm{\tilde{\BB}^{-1}-\BB^{-1}}\Norm{\tilde{\pi}_{\odot}^{-1}} + \Norm{\tilde{y}}\Norm{\tilde{\BB}^{-1}}\Norm{\tilde{\pi}_{\odot}^{-1} - \pi_{\odot}^{-1}}\\
&\leq \zeta^{-42k^2(\log k + 1)-14k^2-20k-19}\pi_{\min}^{-2\log k -8}\eps
\end{align*}

\end{proof}

\section{The core stability bound} \label{sec:condition}
Recall we have set $\beta \coloneqq (\zeta/2)^{k-1}/3k^2 \geq \zeta^{2k}2^{-k}\geq \zeta^{3k}$. 

\begin{theorem} \label{core-stab} 
Let $S$ be a set of $k-1$ $\zeta$-separated vectors $\mm_1,\ldots,\mm_{k-1}$. Then there exists a set 
$J\subseteq 2^S$, $|J|=k$, such that 
 $\sigma_k(\MM[J]) \geq \beta^k2^{-3k/2}k^{-3/2}$ and $\sigma_{\max}(\MM[J]) \leq k$. The first row of $\MM[J]$ is $\One$ (corresponding to $\emptyset \in J$). \end{theorem}

We start with some definitions. 
\begin{definition}
For any subspace $U\subseteq \R^n$, let $\Proj{U}$ denote the orthogonal projection onto $U$ and let $\ProjPerp{U}$ denote the orthogonal projection onto the orthogonal complement subspace. Then $I = \Proj{U} + \ProjPerp{U}$. 
\end{definition}
In the following, $\Norm{\cdot}$ will always denote the $2\to 2$ operator norm of a matrix. 
\begin{definition}[$U$-operator norm]
Given a subspace $U\subseteq \R^n$, we can define the $U$-operator norm of a matrix $C \in \R^{m\times n}$, denoted $\Norm{C}_{U}$, as follows: \[\Norm{C}_U \coloneqq \max_{0\neq u\in U}\Norm{Cu}/\Norm{u}.\]
\end{definition}

\paragraph{Bounding the condition number of $\MM[2^S]$.}
Most of our work will go into lower bounding the $k$'th singular value of $\MM[2^S]$. This is where the $\zeta$-separation condition is essential.
(Even the non-quantitative result that $\MM[2^S]$ has full column rank is not trivial; indeed, the result of~\cite{TahmasebiMM18} that sources can be identified from $2k-1$ separated bits, is implied by the non-quantitative version of this section, omitted here, which employs the same approach but is considerably shorter.)

\begin{lemma}\label{subadd} Let $U$ be any subspace of $\R^n$ and let $C \in \R^{n\times n}$. For all $n \geq 1$, $\Norm{\ProjPerp{U} C^n}_U
\leq n \|C\|^{n-1} \|\ProjPerp{U} C \|_U $.
 \end{lemma} \begin{proof}  
 We need to show that
 $\forall n\geq 2 \; \max_{0\neq u\in U} \Norm{\ProjPerp{U} C^n u} / \Norm{u}
\leq n \Norm{C}^{n-1} \max_{0\neq u\in U} \Norm{\ProjPerp{U} C u} / \Norm{u}$.
Using that $\Norm{\Proj{U}},\Norm{\ProjPerp{U}}\leq 1$, we have \[\Norm{\ProjPerp{U} C^n u} =
\Norm{\ProjPerp{U} C (\Proj{U}+\ProjPerp{U}) C^{n-1} u} \leq
\Norm{\ProjPerp{U} C \Proj{U} C^{n-1} u} + \Norm{\ProjPerp{U} C \ProjPerp{U} C^{n-1} u} .\]
So (using $\Norm{C^{n-1}}\leq \Norm{C}^{n-1}$
and $\Norm{\Proj{U}}\leq 1$
 in the first term,
and $\Norm{\ProjPerp{U}}\leq 1$ in the second term)
\[\max_{0\neq u\in U} \Norm{\ProjPerp{U} C^n u} / \Norm{u} \leq
\Norm{C}^{n-1} \max_{0\neq u\in U} \Norm{\ProjPerp{U} C u} / \Norm{u} 
+ \Norm{C} \max_{0\neq u \in U} \Norm{\ProjPerp{U} C^{n-1} u}/ \Norm{u}
\]
and applying induction this is \[ \leq (\Norm{C}^{n-1} + (n-1)\Norm{C}^{n-1}) \max_{0\neq u\in U} \Norm{\ProjPerp{U} C u} / \Norm{u} . \qedhere \]
\end{proof}

Recall from Section~\ref{sec:alg} that the coefficient norms of the interpolation polynomials are
$\Norm{p_{v,i}}_c=\sum_0^{k-1} j \Abs{p_{v,i,j}}$.

\begin{lemma} $\Norm{\ProjPerp{U} P_{(i)}}_U 
\leq \Norm{p_{v,i}}_c  \cdot \Norm{\ProjPerp{U} v_\odot}_U $.
 \label{odot-to-projection}
\end{lemma}\begin{proof} We are to show that
$\max_{0 \neq u \in U} \Norm{\ProjPerp{U} P_{(i)} u} / \Norm{u} \leq \Norm{p_{v,i}}_c \max_{0 \neq u \in U} \Norm{\ProjPerp{U} v_\odot  u}/ \Norm{u} $. We have
\[ \max_u \Norm{\ProjPerp{U} P_{(i)} u} / \Norm{u} =
\max_u \Norm{\ProjPerp{U} p_{v,i}(v_\odot) u}/ \Norm{u} \leq 
\sum_j \max_u \Abs{p_{v,i,j}} \cdot \Norm{\ProjPerp{U} (v_\odot)^j u}/ \Norm{u}. \] Note that $\Norm{v_\odot}\leq 1$. So applying Lemma~\ref{subadd}: 
\[\max_{0 \neq u \in U} \Norm{\ProjPerp{U} P_{(i)} u} / \Norm{u}
\leq \sum_j \Abs{p_{v,i,j}} j \max_{0 \neq u \in U}  \Norm{\ProjPerp{U} v_\odot u}/ \Norm{u} 
=\Norm{p_{v,i}}_c \max_{0 \neq u \in U} \Norm{\ProjPerp{U} v_\odot u}/ \Norm{u} . \qedhere \] \end{proof}

\begin{lemma} Let $\One \in U \subsetneq \rset^k$ (strict containment). Then there is an $i$ s.t.\
$\|\ProjPerp{U} P_{(i)} \One/\sqrt{k}\| > 1/3k^2$. \label{PiGrowth}
\end{lemma}
\begin{proof}
Let $\ep=1/3k^2$ and suppose to the contrary that $\|\ProjPerp{U} P_{(i)}\One/\sqrt{k}\|\leq \ep$ for all $i$.

Fix some $u\in U$ and let $v_i=P_{(i)} u$ and $u_i=\Proj{U} v_i = \Proj{U} P_{(i)} u$. Then 
since $\sum P_{(i)}=I$, $\sum u_i= \Proj{U} \sum P_{(i)}u=u$.
 
Taking $u=\One/\sqrt{k} \in U$, we find $\Norm{v_i}^2 \geq 1/k$ for all $1\leq i \leq k$. (On the other hand of course $\Norm{v_i} \leq 1$.) By assumption $\ep \geq \Norm{\ProjPerp{U} P_{(i)} \One/\sqrt{k} } = \Norm{\ProjPerp{U} v_i} = \Norm{\Proj{U} v_i - v_i}$. So $\Norm{\Proj{U}v_i}\geq 1/\sqrt{k}-\ep$. For $i \neq j$,
 $v_i^*v_j = 0$, so $\Abs{(\Proj{U} v_i)^*(\Proj{U} v_j)}
=\Abs{((\Proj{U} v_i-v_i) + v_i)^*((\Proj{U} v_j-v_j) + v_j)} \leq \Abs{v_i^*v_j} + 2 \ep+\ep^2=2\ep + \ep^2$. So we have $k$ vectors $\Proj{U} v_i$ in $U$, each of length $\geq 1/\sqrt{k}-\ep$, with inner products $\leq 2\ep+\ep^2$ in absolute value. Their Gram matrix has diagonal entries $\geq  (1/\sqrt{k}-\ep)^2\geq 1/k - 2\ep +\ep^2$ and off-diagonal entries $\leq 2\ep+\ep^2$ in absolute value. Since $\ep= 1/3k^2$, this matrix is strictly diagonally dominant and therefore nonsingular, contradicting the assumption that $U$ has dimension $<k$.
\end{proof}

\begin{lemma}  Let $\One \in U \subset \rset^k$. Then
$\|\ProjPerp{U} v_\odot \|_U > \beta$. \label{Ugrowth} \end{lemma}
\begin{proof}
By Lemma~\ref{PiGrowth}, there is an $i$ s.t.\
$\Norm{\ProjPerp{U} P_{(i)}}_U > 1/3k^2$. Applying Lemma~\ref{odot-to-projection} to this $i$, we have that 
$\Norm{\ProjPerp{U} v_\odot}_U \geq \frac{1}{3k^2 \Norm{p_{v,i}}_c}$. 
Now we need an upper bound on $\Norm{p_{v,i}}_c$. Recall $p_{v,i}(x)=\frac{\prod_{j\neq i}(x-\lambda_j)}{\prod_{j\neq i}(\lambda_i-\lambda_j)}$. Simply by upper-bounding all $\lambda_j$ by $1$ and lower-bounding all separations by $\zeta$,
we have the bound 
$\Norm{p_{v,i}}_c \leq \zeta^{1-k} \sum_{\ell'=0}^{k-1} \ell' \binom{k-1}{\ell'}\leq (k-1)(2/\zeta)^{k-1}$. 
\end{proof}

\begin{lemma} Fix $\zeta$-separated vectors $\mm_1,\dotsc,\mm_{k-1}$. Then there exists a $k\times k$ matrix $V$, with rows $v_1,\dotsc,v_k$, such that:
\begin{enumerate}
\item The row $v_1 = \One/\sqrt{k}$, and for $\ell \geq 1$, $v_{\ell+1} \coloneqq \mm_\ell \odot u$, where $u$ is a unit vector in $U_\ell \coloneqq \Span\Set{v_{1},\dotsc, v_{\ell}}$.
\item Any unit vector in $U_\ell$ can be formed as a linear combination of the rows in $[\mm_R]_{R\subseteq [\ell]}$ with coefficients bounded in maximum magnitude by $\beta^{-\ell}$. 
\end{enumerate} \label{lem:main}
\end{lemma} 
\begin{proof}
We'll use induction over $\ell$. For $\ell > 1$, we'll apply Lemma~\ref{Ugrowth} to find a unit vector $u^{*}\in U_{\ell - 1}$ such that $\Norm{\ProjPerp{U} (\mm_{\ell} \odot u^{*})} > \beta$. We'll set $v_{\ell} \coloneqq \mm_{\ell} \odot u^{*}$. Then $U_{\ell} = \Span(U_{\ell-1},\ProjPerp{U}(\mm_{\ell} \odot u^{*}))$. 
Now $\Norm{\mm_{\ell} \odot u^{*}} \leq 1$. Any vector $u \in U_{\ell}$ can be written as $c_1 u^{(\ell-1)} + c_2(\mm_\ell \odot u^{*})$ for $u^{(\ell-1)}\in U_{\ell-1}$ a unit vector and where $\Abs{c_2} \leq \beta^{-1}$ and $\Abs{c_1} \leq \beta^{-1}$. (This is just the operator norm of the inverse of $\begin{bsmallmatrix} 1 & 0\\ \sqrt{1-\beta^2} & \beta\end{bsmallmatrix}$.)
By the induction hypothesis, we can expand $u^{(\ell-1)}$ and $u^{*}$ in terms of ${\mm_{S}}_{R\subseteq [\ell-1]}$ as follows: 
\begin{align*}
u 
&= c_1 u^{(\ell-1)} + c_2(\mm_\ell \odot u^{*})\\
&= c_1 \sum_{R\subseteq[\ell-1]}\alpha_R \mm_R + c_2 \mm_\ell \sum_{R\subseteq [\ell-1]} \alpha'_R \mm_R\\
&= \sum_{R\subseteq[\ell-1]}\Paren{c_1\alpha_R \mm_R + c_2 \alpha'_R\mm_{R\cup\Set{\ell}}}
\end{align*}
where $\Norm{\alpha}_{\infty},\Norm{\alpha'}_\infty \leq \beta^{\ell-1}$. The claim follows immediately.
\end{proof}

\begin{corollary}
For any unit column vector $z\in \R^k$, the matrix $\MM[2^S]$ satisfies $\Norm{\MM[2^S] z}_{\infty} \geq \beta^{k}2^{-k}$. 
\end{corollary}
\begin{proof}
By Lemma~\ref{lem:main}, we know that we can write $z^\tpose = \lambda^\tpose\MM$ where $\lambda \in \R^{2^k}$ and $\Norm{\lambda}_{\infty} \leq \beta^{-k}$. Thus, $1 = \Norm{z}^2 = \sum_{R\subseteq S} \lambda_R\mm_R z$. There must be some $R \subseteq S$ for which $\Abs{\lambda_R\mm_Rz} \geq 1/2^k$. Since $\Abs{\lambda_R} \leq \beta^{-k}$ we immediately get that $\Norm{\MM z}_{\infty} \geq \Abs{\mm_Rz} \geq \beta^{k}2^{-k}$. 
\end{proof}
For a subset $T$ with $\Card{T} > k-1$, the bound on the largest singular value of $\MM[2^T]$ increases to $k2^{\Card{T}}$ while the lower bound remains unchanged.

\begin{corollary} \label{sigk2S}
$\sigma_{\max}(\MM[2^S]) \leq k2^{k-1}$ and $\sigma_{k}(\MM[2^S]) \geq \beta^{k}2^{-k}/k$.
\end{corollary}
\begin{proof}
The largest singular value of $\MM[2^S]$ is easily upper bounded by $k2^{k-1}$. Lemma~\ref{lem:main} gives the bound on $\sigma_k$. 	
\end{proof}

\paragraph{Bounding the condition number of a $k\times k$ submatrix of $\MM[2^S]$.}
We can now use the following result from \cite{FOS08} to find a $k\times k$ submatrix that is similarly well-conditioned. 
\begin{lemma}[Corollary~6 in \cite{FOS08}] \label{lem-fos}
Let $A\in \R^{k\times n}$ with $k < n$, and let $\sigma_k(A) \geq \eps$. Then there exists a subset of the columns $J\subseteq [n]$ with $\Card{J} = k$ such that $\sigma_k(A_J) \geq \eps / \sqrt{k(n-k) + 1}$. 
\end{lemma}

\begin{proofof}{Theorem~\ref{core-stab}}
The upper bound is trivial since all entries are in $[0,1]$. The lower bound
follows by applying Lemma~\ref{lem-fos} to Corollary~\ref{sigk2S}.
\end{proofof}

\appendix 
\section{Miscellaneous Proofs}
\begin{lemma}
For an invertible $n\times n$ matrix $M$ and a perturbed matrix $\tilde{M}$, if $\Norm{\tilde{M} - M} = \eps \leq \sigma_n(M)/2$, then \[ \Norm{\tilde{M}^{-1} - M^{-1}} \leq 2\Norm{M^{-1}}^2\eps,\quad\text{ and } \Norm{\tilde{M}^{-1}} \leq 2\Norm{M^{-1}}. \] \label{lem:op norm for inv and perturb inv}
\end{lemma}
\begin{proof}
First, we observe that \[\Norm{\tilde{M}^{-1}} = \frac{1}{\sigma_n(\tilde{M})} \leq \frac{1}{\sigma_n(M) - \sigma_n(M)/2} \leq 2\Norm{M^{-1}}.\] 

We use the identity $\tilde{M}^{-1} - M^{-1} = \tilde{M}^{-1}\Paren{M - \tilde{M}}M^{-1}$.
\[\Norm{\tilde{M}^{-1} - M^{-1}} = \Norm{\tilde{M}^{-1}\Paren{M - \tilde{M}}M^{-1}} \leq 2\Norm{M^{-1}}^{2}\Norm{M - \tilde{M}}.\]

\end{proof}

\newpage
\bibliographystyle{alpha}
\bibliography{refs}
\end{document}